\documentclass[letterpaper]{article} %
\usepackage{aaai2027}  %
\usepackage[hyphens]{url}  %
\usepackage{graphicx} %
\usepackage{natbib}  %
\usepackage{caption} %
\usepackage{algorithm}
\usepackage{algorithmic}

\newcommand{\Q}{\mathbb{Q}}

\newcommand{\R}{\mathbb{R}}

\usepackage{amssymb}
\usepackage{amsmath}

\usepackage{amsthm}
\usepackage{xcolor}
\usepackage{stmaryrd}

\newtheorem{theorem}{Theorem}
\newtheorem{lemma}{Lemma}
\newtheorem{proposition}{Proposition}
\newtheorem{definition}{Definition}

\newcommand{\DS}{\mathsf{DS}}

\newcommand{\Inf}{\mathit{Inf}}
\newcommand{\Act}{\mathit{Act}}

\providecommand{\R}{\mathbb{R}}
\providecommand{\Q}{\mathbb{Q}}

\providecommand{\DS}{\mathsf{DS}}

\providecommand{\Buchi}{\textsf{B\"uchi}}

\usepackage{booktabs}

\usepackage{tikz} 
\usetikzlibrary{arrows.meta}

\usepackage{newfloat}
\usepackage{listings}
\DeclareCaptionStyle{ruled}{labelfont=normalfont,labelsep=colon,strut=off} %
\floatstyle{ruled}
\newfloat{listing}{tb}{lst}{}
\floatname{listing}{Listing}

\usepackage{booktabs}

\title{Categorizer Automata for Discounted-Sum Payoffs}
\author{
    Nathalie Bertrand\textsuperscript{\rm 1},
    Pranav Ghorpade\textsuperscript{\rm 2},
    Senthil Rajasekaran\textsuperscript{\rm 3},
    Sasha Rubin\textsuperscript{\rm 2},
    Moshe Vardi\textsuperscript{\rm 4}
}
\affiliations{
    \textsuperscript{\rm 1}Univ Rennes, Inria, CNRS, IRISA, France\\
    \textsuperscript{\rm 2}University of Sydney, Australia\\
    \textsuperscript{\rm 3}Université Libre de Bruxelles,  Belgium\\
    \textsuperscript{\rm 4} Rice University, USA\\
}

\begin{document}

\maketitle

\begin{abstract}
Categorizing continuous data into discrete bins is a fundamental operation in artificial intelligence. We introduce the \emph{categorizer automaton}, a deterministic automaton that reads an infinite sequence of rewards and identifies which of finitely many bins contains its discounted sum. Categorizer automata generalize comparator automata, the special case of two bins, which have already proven useful in quantitative synthesis. Our main technical contribution is the construction of a categorizer automaton whose state space is linear in the number of bins, rather than exponential as obtained by a cross-product of comparator automata.  We then apply categorizer automata to Markov decision processes, where they allow one to synthesize policies that maximize the expected utility of a discounted-sum payoff for utility functions that may be \emph{discontinuous}. For piecewise-constant utility functions, the resulting algorithm is exact and runs in pseudo-polynomial time. For piecewise-Lipschitz utility functions, a class that includes any utility with bounded slope between finitely many jumps, it again runs in pseudo-polynomial time and yields an $\varepsilon$-optimal policy. We also show that the synthesis problem considered is \textsf{PSPACE}-hard already for piecewise-constant utilities.

\end{abstract}

\section{Introduction}

The fields of verification and sequential decision-making often involve assigning a numerical payoff to an infinite-length execution of a system. One of the most popular ways to reason about such executions is through the discounted-sum payoff~\cite{DBLP:conf/icalp/AlfaroHM03,Puterman94}, which aggregates an infinite sequence of rewards along an execution through a discounted sum. For a discount factor $d>1$, the discounted-sum payoff of an infinite reward sequence $r_0 r_1 \ldots$ is
\(\DS_d(r_0 r_1 \ldots) := \sum_{i=0}^{\infty} d^{-i} r_i\).

Representing numerical properties of discounted-sum payoffs with finite-state automata allows synthesis and verification problems involving such properties to be handled using efficient automata-theoretic techniques.
A notable example is the development of the \emph{comparator automaton}~\cite{DBLP:conf/fossacs/BansalCV18,bansal2022comparator}. A comparator automaton reads an infinite reward sequence $w$ and accepts it iff
$\DS_d(w) \bowtie t$, for a fixed rational threshold $t$ and a comparison relation $\bowtie$ such as $>$ or $\leq$. Comparator automata have since been used for 
synthesis for \emph{satisficing} (synthesize a ``good enough'' policy rather than an optimal one~\cite{simon1956rational}) objectives in two-player discounted-sum games~\cite{DBLP:conf/tacas/BansalCV21}, synthesis that combines discounted-sum payoffs with temporal specifications~\cite{rational-discount-AAAI}, and Nash-equilibrium realizability under discounted-sum payoffs in deterministic multi-agent systems~\cite{rajasekaran_multi-agent_2023}.

In many settings, however, comparing the payoff with a single threshold does not provide enough information. For example, a decision maker may assign different utilities to losses, moderate gains, and large gains, or may wish to simultaneously track the payoff up to a prescribed accuracy. 

In this paper, we introduce the \emph{categorizer automaton}, a generalization of comparator automaton. Instead of reasoning about a single threshold like the comparator automaton, the categorizer automaton simultaneously considers several thresholds that divide the real line $\mathbb{R}$ into a series of disjoint intervals $I_1, \ldots, I_n$. The problem then becomes
\emph{categorizing} an infinite reward sequence by the interval that its discounted sum lies in, as opposed to \emph{comparing} its discounted sum to a single threshold. That is, a categorizer automaton $\mathcal A$ reads an infinite reward sequence and identifies the unique interval containing its discounted sum.  It does this based on which states repeat infinitely often along the run, a property that is captured by the automata-theoretic B\"uchi condition. %

The comparator automaton is the special case of $n=2$ intervals. A single boundary point splits the real line in two, and the automaton computes one bit about the payoff: whether $\DS_d(w)$ lies above or below the threshold. The use of finer partitions reveals more information about the payoff. In particular, because the discounted-sum payoff is bounded when the rewards are bounded, for every precision $\varepsilon>0$ the bounded range of possible discounted sums can be divided into finitely many intervals of length at most $\varepsilon$ (with the rest of
$\mathbb{R}$ covered by two unbounded intervals) so that the resulting categorizer automaton identifies an interval of width at most $\varepsilon$ containing $\DS_d(w)$.
 
A natural way to build a categorizer automaton is to take one comparator automaton per boundary point and construct their cross-product. This works, but the state space grows exponentially with the number $n$ of intervals. This is prohibitively large when $n$ is part of the input, as in the approximation regime above, where reaching precision $\varepsilon$ forces a number of intervals that grows as $\varepsilon$ shrinks.

Our main technical contribution is a construction of the categorizer automaton that has a state space that is linear in the number $n$ of intervals. As in the comparator automata setting, we assume that rewards are bounded integers, the boundary points are rational, and the discount factor $d$ is an integer. 
Under these assumptions, for every finite partition of $\R$ into intervals, our construction yields a categorizer automaton whose state space is linear in $n$ and polynomial in the numeric value of the largest reward and the largest denominator among the boundary points when written in lowest terms. The main idea behind this construction is that the per-boundary comparisons are highly dependent on one another, rather than independent as a cross-product construction via comparator automata would suggest. Exploiting this dependence is what allows one to avoid the exponential blowup.

\paragraph{Applications}
A categorizer automaton can be composed with an arbitrary finite-state model whose executions generate bounded reward sequences, including transition systems, games, Markov decision processes (MDPs), and stochastic games. In the resulting composition, the interval
containing the discounted-sum payoff is represented by a B\"uchi condition. %
Thus, via categorizer automata, objectives depending on the exact payoff interval, an approximate payoff value, or both, can be reduced to automata-theoretic objectives.

We apply categorizer automata to MDPs, where they allow one to synthesize policies that maximize the expected utility of a discounted-sum payoff under utility functions that may be \emph{discontinuous}.
Utility functions provide a standard way to express how a decision maker values different outcomes, allowing the model to reflect attitudes toward risk rather than only the payoff~\cite{mas1995microeconomic}. Abrupt changes in utility (discontinuities) arise naturally when meeting a target or crossing a safety threshold. %
Using categorizer automata together with existing automata-theoretic techniques~\cite{DBLP:journals/tac/CourcoubetisY98,lmcs:990}, we obtain, for integer discount factors, an exact algorithm for piecewise-constant utilities with finitely many rational discontinuities. The algorithm computes the optimal expected utility and synthesizes a finite-memory policy attaining it. For utilities that are Lipschitz continuous on each of finitely many intervals but may be discontinuous at their boundaries, we build on the piecewise-constant result to compute the optimal value up to an additive error of $\varepsilon$ and synthesize an $\varepsilon$-optimal policy.
Both algorithms run in pseudo-polynomial time: their running times are polynomial in certain numerical parameter values, but not in the lengths of their binary encodings. This dependence is unlikely to be avoidable, as we also show that the synthesis problem is \textsf{PSPACE}-hard even for piecewise-constant utilities. 

We remark that the previous applications of comparator automata discussed above have focused on
non-probabilistic models. Our application to MDPs extends the literature to stochastic settings.

\paragraph{Related Work} 
The classical discounted-sum payoff objective maximizes the expectation $\mathbb E^{\theta}_{\mathcal M}[\DS_d]$ over policies $\theta$ in an MDP $\mathcal M$. This corresponds to the linear utility $u(x)=x$ and can be solved using the Bellman equations~\cite{Puterman94}. In general, for a nonlinear utility $u$, the objective $\mathbb E^{\theta}_{\mathcal M}[u(\DS_d)]$ no longer admits Bellman-style equations on the original state space of the MDP~\cite{General-Utility}.
Existing algorithmic approaches typically exploit special structure in the utility, as for exponential utility~\cite{INFORMS2014}, or augment the state space with accumulated reward and apply discretizations or finite-horizon approximations~\cite{DBLP:journals/corr/abs-2311-13589}.
Such approximations naturally apply  to continuous utilities: truncating the discounted sum changes the payoff by a vanishing amount, and uniform continuity over the bounded payoff range then bounds the resulting change in the utility.

This argument fails for discontinuous utilities. Near a discontinuity, an arbitrarily small change in the payoff can produce a large change in utility. The simplest example is a threshold utility, defined by $u(x)=0$ for $x<v$ and $u(x)=1$ for $x\geq v$, whose expected value is precisely the probability that the discounted-sum payoff is at least $v$. Threshold utilities have been studied by~\citet{white_minimizing_1993} and~\citet{wu_minimizing_1999}, who characterize Bellman-style equations on an extended, possibly \emph{infinite}, state space and investigate the conditions under which optimal policies exist. 
{From an algorithmic perspective, they have been studied in the finite-horizon setting~\cite{xu_probabilistic_2011}. For the infinite-horizon setting,~\citet{randour_percentile_2015} study percentile queries through an $\varepsilon$-gap formulation, which permits instances sufficiently close to the payoff threshold to remain unresolved. Thus, prior algorithms either restrict the horizon or relax the threshold comparison, whereas, for integer discount factors, our algorithm solves the exact problem.}

\section{Categorizer Automata}
\label{sec:dsqa}
This section introduces \emph{categorizer  automata} for discounted-sum payoffs (Definition~\ref{def:dsqa}), characterizes when categorizer automata exist (Proposition~\ref{thm:dsqa}) and for the cases when they exist, establishes a size bound (Theorem~\ref{thm:size}).  

\subsection{Preliminaries}
\label{sec:prelim}

\paragraph{Words} 
Given a finite \emph{alphabet} $\Sigma$ of symbols, a \emph{word} (resp. \emph{finite word}) $w$ is an infinite (resp. finite) sequence of symbols. We denote the length of a finite word as $|w|$, and refer to the unique finite word of length $0$ as $\varepsilon$. For a word (resp. finite word) $w$ and index $1 \le i$ (resp. $1 \leq i \leq |w|$) we denote: $w[i]$ as the $i$-th symbol of $w$, $w[1..i]$ as the length $i$ prefix of $w$  and $w[i..]$ as the suffix starting at position $i$. For a finite word $w$ and a finite (resp. infinite) word $w'$, we denote $w \cdot w'$ as the finite (resp. infinite) word formed by concatenating $w$ and $w'$. We write $\Sigma^\omega$ (resp. $\Sigma^*$) for the set of all words (resp. finite words). 

\paragraph{Transition System} 
A \emph{(deterministic) transition system} is a tuple $\mathcal{T} = (Q, \Sigma, q_{\mathrm{init}}, \delta)$, where $Q$ is a finite set of \emph{states}, $\Sigma$ is a finite alphabet, $q_{\mathrm{init}} \in Q$ is the \emph{initial state}, and $\delta : Q \times \Sigma \to Q$ is a \emph{transition function}. The \emph{run} of $\mathcal{T}$ on a word $w \in \Sigma^\omega$ is the sequence of states $\rho = \rho[1] \rho[2]\cdots$ with $\rho[1] = q_{\mathrm{init}}$ and $\rho[k{+}1] = \delta(\rho[k], w[k])$ for all $k \ge 1$.  We denote $\Inf(\rho)$ for the set of states that occur infinitely often in $\rho$.
Given transition systems $\mathcal{T}_i=(Q_i,\Sigma,q_{\mathrm{init}}^i,\delta_i)$, for $1\leq i\leq n$, over the same alphabet $\Sigma$, their \emph{cross-product} is the transition system whose states are tuples $(q_1,\ldots,q_n)\in Q_1\times\cdots\times
Q_n$, whose initial state is $(q_{\mathrm{init}}^1,\ldots,q_{\mathrm{init}}^n)$, and whose transition on a symbol $\sigma$ maps $(q_1,\ldots,q_n)$ to $(\delta_1(q_1,\sigma),\ldots,\delta_n(q_n,\sigma))$.

\paragraph{B\"uchi Acceptance}
A \emph{B\"uchi acceptance condition} is defined by a set $F \subseteq Q$ of states. A run $\rho$ \emph{satisfies} the B\"uchi acceptance condition $F$  if $\Inf(\rho) \cap F \ne \emptyset$, i.e. if at least one state of $F$ occurs infinitely often in the run $\rho$. A \emph{(deterministic) B\"uchi automaton} is a pair $\mathcal{A} = (\mathcal{T}, F)$ of a transition system and a B\"uchi acceptance. Its \emph{language}, called the language recognized by $\mathcal{A}$, is the set of words $w \in \Sigma^\omega$ whose run on $\mathcal{T}$ satisfies $F$. A language (set of words) is \emph{$\omega$-regular} if it is recognized by some (possibly nondeterministic) B\"uchi automaton. 
Deterministic B\"uchi automata are effectively closed under intersection: given deterministic B\"uchi automata with states $Q_1$, $Q_2$ recognizing languages $L_1, L_2$, one can construct a deterministic B\"uchi automaton with $2|Q_1||Q_2|$ states recognizing $L_1 \cap
L_2$~\cite{PMC-book}.

\subsection{Discounted-sum and Categorizer Automaton}
\label{sec:ds}

\paragraph{Discounted-Sum}
Given an integer \emph{alphabet bound} $\mu \ge 1$, we denote $\Sigma_\mu$ as the finite alphabet consisting of all the integers from $-\mu$ to $+\mu$. For discount factor $d >1$, the discounted sum of a word $w \in \Sigma_\mu^\omega$ (resp. $w \in \Sigma_\mu^*$) is $\DS_d(w) =\sum_{i=1}^\infty w[i]/d^{i-1}$ (resp. $\DS_d(w) =\sum_{i=1}^{|w|} w[i]/d^{i-1}$).

\paragraph{Binning of $\R$} A \emph{binning} of $\mathbb{R}$ is a finite partition $B = \{I_1, \dots, I_n\}$ of $\mathbb{R}$ into $n\ge2$ nonempty, disjoint intervals, indexed so that every element of $I_i$ is smaller than every element of $I_{i+1}$.  The \emph{endpoints} of $B$ are the finite infimum and supremum of its intervals:
the singleton $\{t\}$ has the one endpoint $t$, whereas $[t,t')$, $[t,t']$,
$(t,t')$, and $(t,t']$ each have the two endpoints $t$ and $t'$. The outer
intervals $I_1$ and $I_n$ are unbounded and so have one endpoint each. Writing $t_1 < \dots < t_m$ for the endpoints of $B$ in increasing order, we have $m \le n-1$, and each $I_i$ is either the singleton $\{t_j\}$ or an interval with endpoints $t_j$ and $t_{j+1}$, for some $0 \le j \le m$, where $t_0 = -\infty$ and $t_{m+1} = +\infty$. A binning is \emph{rational} if all its endpoints are rational.

\begin{definition}[Categorizer automaton]
\label{def:dsqa}
Let $\mu \geq 1$ be an integer alphabet bound, let $d>1$ be a rational discount factor, and let $B=\{I_1,\dots,I_n\}$ be a rational binning of $\mathbb{R}$. A \emph{categorizer automaton} for $(\mu,d,B)$ is a tuple \(\mathcal{A}=(\mathcal{T},F_1,\dots,F_n),\) where \(\mathcal{T}=(Q,\Sigma_\mu,q_{\mathrm{init}},\delta)\) is a deterministic transition system and each $F_i\subseteq Q$ is a B\"uchi acceptance condition, such that, for every word $w\in\Sigma_\mu^\omega$ and every $1\leq i\leq n$, it holds that $\DS_d(w) \in I_i$ if and only if the run of $\mathcal{T}$ on $w$
satisfies $F_i$.
\end{definition}
In other words, a categorizer  automaton for $(\mu, d, B)$ is one transition system equipped with $n$ B\"uchi conditions, one per interval of $B$ such that for each $i$, the language of the B\"uchi automaton $(\mathcal{T}, F_i)$ is  $\{w \in \Sigma_\mu^\omega \mid \DS_d(w) \in I_i\}$. 

\paragraph{Relation to Comparator Automata}
Categorizer automata generalize comparator automata~\citep{bansal2022comparator}. Given an integer alphabet bound $\mu$, a rational discount factor $d>1$, a rational threshold $t$, and a relation ${\bowtie}\in \{<,>,\leq,\geq,=,\neq\}$, a \emph{comparator automaton} is a deterministic B\"uchi automaton that recognizes $\{w\in\Sigma_\mu^\omega \mid \DS_d(w)\bowtie t\}.$ \citet{bansal2022comparator} show that, for every relation $\bowtie$, such automata exist for every rational threshold $t$ if and only if $d$ is an integer.

We now use this result to characterize when categorizer automata exist for every rational binning. 
To see that the integrality of $d$ is necessary, suppose that $d$ is not an integer. The preceding result gives a rational threshold $t$ for which no comparator automaton recognizes $\{w\in\Sigma_\mu^\omega \mid \DS_d(w)<t\}$. A categorizer automaton $(\mathcal T,F_1,F_2)$ for the rational binning $B_t=\{(-\infty,t),[t,+\infty)\}$ would make $(\mathcal T,F_1)$ precisely such a comparator automaton, yielding a contradiction.

Conversely, suppose that $d$ is an integer, and fix a rational binning $B=\{I_1,\ldots,I_n\}$. We construct a categorizer automaton for $(\mu,d,B)$ using comparator automata. For each $i$, first construct a deterministic B\"uchi automaton $\mathcal{A}_i=(\mathcal{T}_i,G_i)$ that recognizes the language $L_i=\{w\in\Sigma_\mu^\omega \mid \DS_d(w)\in I_i\}.$ Such an automaton can be constructed from comparator automata: depending on the form of $I_i$, the language $L_i$ is either a comparator language or the intersection of two comparator languages. %
Taking the cross-product of $\mathcal T_1,\ldots,\mathcal T_n$ and letting $F_i$ consist of the product states whose $i$th components belong to $G_i$ yields a categorizer automaton for $(\mu,d,B)$. Hence, we obtain the following proposition.

\begin{proposition}[Existence of categorizer automata]
\label{thm:dsqa}
Let $\mu\geq 1$ be an integer, and let $d>1$ be rational. Categorizer automata for $(\mu,d,B)$ exist for every rational binning $B$ if and only if $d$ is an integer.
\end{proposition}

For non-integer $d$, the situation is more complex: it is an open problem to characterize the individual rational thresholds that admit comparator automata. For integer $d$, although the cross-product argument above establishes the existence of a categorizer automaton, it produces one whose size is exponential in the number of bins. In the next section we present a construction that avoids this blowup.

\subsection{Construction That Avoids Exponential Blowup}
In this section, we present a direct construction of a categorizer automaton whose size is linear in the number of bins. 

Throughout, \(q_B\) denotes the largest denominator, in lowest terms, of endpoints of \(B\).

\begin{theorem}
\label{thm:size}
Let $\mu\geq 1$ be an integer alphabet bound, let $d>1$ be an integer
discount factor, and let $B=\{I_1,\dots,I_n\}$ be a rational binning. There exists a categorizer
automaton for $(\mu,d,B)$ with \(\mathcal{O}(n\mu q_B(1+\log_d(\mu q_B)))\) many states and it can be constructed in $\mathcal{O}(n\mu q_B(1+\log_d(\mu q_B)))$ time.
\end{theorem}
 
The bound is linear in the number of bins and polynomial in the numerical
values of $\mu$ and $q_B$ (thus, it is pseudo-polynomial in $\mu$ and $q_B$).

The remainder of this section proves
Theorem~\ref{thm:size} by an explicit construction. To that end, fix $\mu$, $d$, and $B$ as in
Theorem~\ref{thm:size}. 
Let $t_1<\dots<t_m$ be the endpoints of $B$, set
$t_0=-\infty$ and $t_{m+1}=+\infty$, and write each finite endpoint in
lowest terms as $t_j=p_j/q_j$ with $q_j>0$ so that $q_B=\max_{1\leq j\leq m}q_j$. Define $D=\DS_d(\mu\mu\cdots)=(\mu d)/({d-1}),$ so that $\DS_d(w)\in[-D,D]$ for every $w\in\Sigma_\mu^\omega$.

\paragraph{High-Level Idea}
After every finite word $u$, the states of the categorizer automaton track a \emph{gap value} for each endpoint. Intuitively, the gap value of an endpoint records the discounted sum that a continuation of $u$ must have for the complete word to have a discounted sum exactly equal to that endpoint. Since every continuation has a discounted sum in $[-D, D]$, a gap value outside this interval determines on which side of the endpoint every continuation must lie. If the gap value lies below $-D$ (resp. above $D$), then, for every continuation $w$, the discounted sum of $u \cdot w$ lies above (resp. below) the endpoint. Otherwise, different continuations can still place the discounted sum below, at, or above the endpoint. Tracking how these gap values evolve therefore allows the automaton to locate the discounted sum relative to every endpoint and, consequently, to locate the bin that contains it.

Storing the tuple of all gap values explicitly would produce a state space exponential in the number of endpoints. The key observation behind our construction is that these values are not independent: all gap values can be recovered from the gap value of one endpoint and the length $|u|$. Their pairwise separations also grow exponentially with $|u|$. Consequently, after only logarithmically many symbols, at most one endpoint has a gap value in $[-D,D]$. From that point onward, its gap value and index suffice to determine the position of the discounted sum relative to every endpoint. The construction therefore encodes all gap values compactly by using one gap value, its endpoint index, and, during the short initial phase, the length $|u|$. This encoding gives a construction whose size is linear in the number of bins.

\subsubsection{Gap Values} 
Gap values are a standard tool for tracking a discounted sum relative to a
single threshold~\cite{BokerH14,bansal2022comparator}. 
For a finite word $u \in \Sigma_\mu^*$ and an endpoint $t_j$, the \emph{gap value} of $t_j$ after $u$  is defined by the recurrence
\begin{equation}
\label{eq:gaps}
g_j(\varepsilon) = t_j,
\quad
g_j(u\cdot a) = d \cdot \big( g_j(u) - a \big)
\quad \text{for } a \in \Sigma_\mu.
\end{equation}
An induction on $|u|$ gives the closed form
\begin{equation}
\label{eq:closedform}
g_j(u) = d^{\,|u|} \big( t_j - \DS_d(u) \big).
\end{equation}
Since $\DS_d(u \cdot w) = \DS_d(u) + d^{-|u|} \DS_d(w)$ for every $w \in \Sigma_\mu^\omega$, the gap value is the discounted sum a continuation of $u$ must have in order to reach $t_j$: for ${\bowtie} \in \{<, =, >\}$ we have $\DS_d(u \cdot w) \bowtie t_j$ iff $\DS_d(w) \bowtie g_j(u)$. In particular, if $g_j(u) < -D$ (resp.\ $g_j(u) > D$) then $\DS_d(u \cdot w) > t_j$ (resp.\ $< t_j$) for every continuation $w$; we then say $t_j$ is \emph{resolved to $\bot$} (resp.\ \emph{to $\top$}) after $u$. Otherwise, i.e. if $g_j(u) \in [-D, D]$, we say $t_j$ is \emph{active} after $u$. 

The following two propositions follow from the definition of gap values and are proved in the supplementary material.

\begin{proposition}
\label{prop:gaps}
Let $1 \le j \le m$ and let $w \in \Sigma_\mu^\omega$. Then $\DS_d(w) > t_j$ iff $t_j$ is resolved to $\bot$ after some prefix of $w$; $\DS_d(w) < t_j$ iff $t_j$ is resolved to $\top$ after some prefix of $w$; and $\DS_d(w) = t_j$ iff $t_j$ is active after every prefix of $w$.
\end{proposition}
For indices $1 \le j \le m$ let $G_j = \{x/{q_j} \mid x \in \mathbb{Z}\} \cap [-D, D].$
In the following proposition, P1 bounds the gaps of active endpoints, and P2 justifies the use of the term ``resolved''.

\begin{proposition}\label{prop:integer-d}
    Let $u \in \Sigma_\mu^*$ and let $1 \le j \le m$. Then:
\begin{enumerate}
\item[P1]  if $t_j$ is active after $u$ then $g_j(u) \in G_j$;
\item[P2]  if $t_j$ is resolved to $\bot$ (resp. $\top$) after $u$ then $t_j$ is resolved to $\bot$ (resp. $\top$) after $u\cdot a$, for all $a \in \Sigma_\mu$.
\end{enumerate}
\end{proposition}

We now describe how these facts can be used to construct a categorizer automata.

\paragraph{The Shared Structure of Gap Values}
By Prop.~\ref{prop:gaps} and Prop.~\ref{prop:integer-d}, it suffices for the transition system of the categorizer  automaton for $(\mu, d, B)$ to track, after every finite word, the \emph{status} of every endpoint: its gap value (from the finite set $G_j$) if active, and the value $\bot,\top$ it resolved to otherwise. %
However, storing the status of every endpoint explicitly requires $\prod_j (|G_j| +2)$ states, which is still exponential in $n$. The next lemma shows that the statuses of $m$ endpoints are not independent, which gives the ingredients for a transition system to track all of them simultaneously without incurring an exponential blowup. 

\begin{lemma}
\label{lem:facts}
Let $u \in \Sigma_\mu^*$ and let $K_0 = \big\lfloor \log_d (2D{q_B}^2) \big\rfloor$. Then: %
\begin{enumerate}
\item[P1]  $g_i(u) - g_j(u) = d^{\,|u|} (t_i - t_j)$ for all $1 \le i, j \le m$;
\item[P2] if $|u| > K_0$ then at most one endpoint is active after $u$;
\item[P3] there exist $\ell, r$ with $0 \le \ell \le r \le m$ such that the endpoints $t_1, \dots, t_\ell$ are resolved to $\bot$ after $u$, the endpoints $t_{\ell+1}, \dots, t_r$ are active after $u$, and the endpoints $t_{r+1}, \dots, t_m$ are resolved to $\top$ after $u$. %
\end{enumerate}
\end{lemma}
\begin{proof} P1 follows from Eq.~\ref{eq:closedform}. For P2, let $t_i,t_j$ be the endpoints active after $u$ i.e. $g_i(u), g_j(u) \in [-D, D]$. It follows that $|g_i(u)- g_j(u)| \le 2D$. Moreover, note that $|t_i -t_j| = |p_iq_j - p_jq_i|/q_iq_j \ge 1/q_iq_j \ge 1/{q_B}^2$. Substituting these two inequalities in P1 of the lemma, we get $d^{\,|u|} /{q_B}^2 \le 2D$. Rearranging and taking the logarithm gives $|u| \le K_0$. For P3, since $t_1 < \dots < t_m$, by Eq.~\ref{eq:closedform}, $g_1(u) < g_2(u) < \dots < g_m(u)$, hence if $t_i$ is resolved to $\bot$ i.e. $g_i(u) < -D$ then $t_j$ for all $j < i$ is also resolved to $\bot$. Similarly for $\top$. 
\end{proof}

\paragraph{The Construction} Formally, the categorizer automaton for $(\mu, d, B)$ is $\mathcal{A} = (\mathcal{T}, F_1, \dots, F_n)$. We encode the status of every endpoint compactly, using a representation according to the number of endpoints active after $u$.  Informally, after reading a finite word $u$, the state of transition system $\mathcal{T}$ records the status of every endpoint, using an encoding that depends on how many endpoints remain active after $u$. Accordingly, the state of $\mathcal{T}$  takes one of three forms: \textbf{(a)} if two or more endpoints are active, the state is a triple $(j, g_j(u), k)$, where $j$ is the index of smallest active endpoint and $k = |u|$. This triple encodes the status of every endpoint: by P1 of Lemma~\ref{lem:facts}, the gap value of every $t_i$ can be recovered, which in turn determines the resolved endpoints and their status; \textbf{(b)} If exactly one endpoint is active, the state is a pair $(j, g_j(u))$, where $j$ is the index of that endpoint. By P3 of Lemma~\ref{lem:facts}, $t_1, \dots, t_{j-1}$ are resolved to $\bot$ and $t_{j+1}, \dots, t_m$ to $\top$; \textbf{(c)} If no endpoint is active, the state is an integer $\ell$, where $t_1, \dots, t_\ell$ are resolved to $\bot$ and $t_{\ell+1}, \dots, t_m$ to $\top$.

\paragraph{Transition System}

Formally, $\mathcal{T} = (Q, \Sigma_\mu,q_{\text{init}}, \delta)$, where $Q = Q_{a} \uplus Q_b \uplus Q_c$ is the union of the states of forms (a), (b), and (c), respectively, with 
\begin{align*}
    Q_{a} &= \big\{\, (j, g, k) \;\big|\; 1 \le j \le m,\ g \in G_j,\ 0 \le k \le K_0 \,\big\}, \\
    Q_b &= \big\{\, (j, g) \;\big|\; 1 \le j \le m,\ g \in G_j \,\big\}, \quad Q_c =\{0, \dots, m \}.
\end{align*}
We now describe the initial state $q_{\text{init}}$ and transition function $\delta$ (the formal definition is given in the supplementary material). 
 The initial state $q_{\text{init}} \in Q$ is determined by the status of the initial gap values, which by~\eqref{eq:gaps} coincide with the endpoints. On input $a \in \Sigma_\mu$, the transition under $\delta$ can be viewed as follows. First, decode the current state into the status of every endpoint. Second,  update these statuses: resolved endpoints are left unchanged (P2 of Prop.~\ref{prop:integer-d}), whereas active endpoints are updated using their new gap values, computed as $g_j' = d(g_j - a)$, as in~\eqref{eq:gaps}. Finally,  re-encode the resulting status as a state of $Q$, choosing the form according to the number of endpoints that remain active.

\paragraph{Acceptance Conditions} Along every run $\rho$ of $\mathcal{T}$, resolved endpoints stay resolved (P2 of Prop.~\ref{prop:integer-d}) and eventually all but at most one of the endpoints become resolved. Hence, for a run $\rho$ of $\mathcal{T}$, there exist $x \ge 0$ such that either (i) for all $y > x$ we have $\rho[y] = (j, g)$ for some $1 \le j \le m$ and $g \in G_j$, or (ii) for all $y > x$ we have $\rho[y] = j$ for some $0 \le j \le m$.
By Prop.~\ref{prop:gaps}, it follows that the run $\rho$ of $\mathcal{T}$ on $w \in \Sigma_\mu^\omega$ visits $\{(j, g) \mid g \in G_j\} \subseteq Q_b$ infinitely often iff $\DS_d(w) = t_j$, and visits the state $j \in Q_c$ infinitely often iff $\DS_d(w) \in (t_j, t_{j+1})$. Thus
\[
F_i = \{ (j, g) \in Q_b\  |\  t_j \in I_i \}  \cup
\{ j \in Q_c \ |\  (t_j, t_{j+1}) \subseteq I_i \},
\]
ensures that the run of $\mathcal{T}$ on $w$ satisfies $F_i$ iff $\DS_d(w) \in I_i$. 

\paragraph{Size} We bound $|Q| = |Q_a| + |Q_b| + |Q_c|$. Since \(|G_j| = 2\lfloor D\ q_j\rfloor + 1 \le 2\lfloor D\ q_B\rfloor + 1\) for every $j$, and $m \le n-1$:%
\begin{align*}
    &|Q_a| \le n\,(2\lfloor D\ q_B\rfloor + 1)\,(K_0+1), \\
    &|Q_b| \le n\,(2\lfloor D\ q_B\rfloor + 1), \text{ and } |Q_c| \le n.
\end{align*}
Thus $|Q|=\mathcal O(nDq_B(1+K_0)).$ Since $d \ge 2$, $D = (\mu d)/(d-1) \le 2\mu$ and hence $D = \mathcal{O}(\mu)$. Moreover, $K_0=\mathcal O(1+\log_d(\mu\ q_B)).$ Substituting these bounds gives $|Q|=\mathcal O(n\mu q_B(1+\log_d(\mu q_B))).$ 
This proves Theorem~\ref{thm:size}: the above bound gives the size while the correctness follows from the discussion above.

\section{Application: Policy Synthesis for Maximizing Expected Utility in MDPs}\label{sec:applications}
In this section we develop the application of categorizer automata to the problem of maximizing expected utility in Markov decision process (MDP) with discounted-sum payoffs. We provide pseudo-polynomial time algorithms and a $\textsf{PSPACE}$-hard lower bound for the synthesis problem under a broad class of utility functions.

\subsection{Preliminaries}
\label{sec:mdp-prelim}
We follow the standard terminology of~\cite{Puterman94,PMC-book}.
\paragraph{MDP} A (discounted reward) \emph{Markov decision process} (MDP) is a tuple $\mathcal M=(S,\Act,s_{\mathrm{init}},P,r, d)$, where $S$ is a finite set of states, $\Act$ is a finite set of actions, $s_{\mathrm{init}}\in S$ is the initial state, $P: S\times \Act \times S \to [0,1] \cap \mathbb Q$ is a transition function, 
$r: S\times \Act \to \mathbb{Z}$ is a reward function, and $d >1$ is an integer discount factor.
For every state $s$ and action $a$ we require
$\sum_{s' \in S} P(s,a,s') \in \{0,1\}$. When this sum equals $1$, we say that
$a$ is \emph{enabled} at $s$, and assume that at least one action is enabled at
every state. We denote $\mu_{\mathcal M}$ as the \emph{reward bound} $\max_{s\in S,\, a \in \Act} |r(s,a)|$ of MDP $\mathcal M$. Let $D_{\mathcal M} = \mu_{\mathcal M} d/(d-1)$. We write $|\mathcal M|$ for the size of the representation of $\mathcal M$: states, actions, and nonzero transitions are counted explicitly, while transition probabilities, rewards, and the discount factor are encoded in binary.

\paragraph{Plays and Policies} A \emph{play} is an infinite sequence $\pi=s_0\,a_0\,s_1\,a_1\cdots$ with
$s_0=s_{\mathrm{init}}$ and $P(s_k,a_k,s_{k+1})>0$ for every $k\ge 0$.
A \emph{history} is a finite prefix of a play ending in a state. Writing $H_\mathcal{M}$ for
the set of all histories,
a \emph{policy} is a function $\theta: H_\mathcal{M} \times \Act \to [0,1]$ such that
$\sum_{a\in \Act}\theta(h,a)=1$ for every history $h$, and $\theta(h,a)=0$ whenever
$a$ is not enabled at the last state of history $h$. A policy is finite-memory if it can be implemented by a finite-state machine~\cite[Def.~10.97]{PMC-book}. 
We write $\Theta_\mathcal{M}$ for the set of all
policies. A policy $\theta$ induces the standard probability measure
$\Pr^\theta_{\mathcal{M}}$ on the set of plays, equipped with the
$\sigma$-algebra generated by the cylinder sets of histories. %

\paragraph{Discounted-Sum Payoff and Expected Utility} Every play $\pi$ determines the \emph{reward word}
$r(s_0,a_0)\,r(s_1,a_1)\dots\in\Sigma_{\mu_{\mathcal M}}^\omega$. 
We denote $\DS_d(\pi)$ as the discounted sum of its reward word and refer to it as the \emph{discounted-sum payoff of $\pi$}.
The map $\pi\mapsto\DS_d(\pi)$ is measurable, so $\DS_d$ is a
random variable under $\Pr^\theta_{\mathcal{M}}$ for every policy $\theta$.  
A utility function \(u:\mathbb{R}\to\mathbb{R}\) assigns a utility value to each discounted-sum payoff. We make the standard technical assumptions that $u$ is Borel measurable and bounded on \([-D_{\mathcal M},D_{\mathcal M}]\). Since every discounted-sum payoff generated by $\mathcal M$ lies in \([-D_{\mathcal M},D_{\mathcal M}]\), \(u(\DS_d)\) is a bounded random variable under \(\Pr_{\mathcal M}^{\theta}\) for every policy $\theta$.  We denote $ V_u^{\theta}(\mathcal M)$ as the expectation $\mathbb{E}_{\mathcal M}^{\theta}\left[u(\DS_d)\right]$, and refer to it as the \emph{expected utility} of \(\theta\). We denote $V_u^{*}(\mathcal M)$ as the supremum of $V_u^{\theta}(\mathcal M)$, over all policies ${\theta\in\Theta_{\mathcal M}}$.
A policy \(\theta\) is called \emph{optimal} if
\(V_u^{\theta}(\mathcal M)=V_u^{*}(\mathcal M)\), and \emph{\(\varepsilon\)-optimal}
if \(V_u^{\theta}(\mathcal M)\geq V_u^{*}(\mathcal M)-\varepsilon\). 

\paragraph{B\"uchi Events} Let $Z \subseteq S$ be a set of states. We denote $\Buchi(Z)$ to be the event consisting of the plays that visit $Z$ infinitely often. 

\subsection{Problem Statement and Overview}
In this subsection, we formalize the synthesis problem, state the results, and give high level ideas used to establish them.
\paragraph{Piecewise-Lipschitz functions}
Given a rational binning \(B=\{I_1,\dots,I_n\}\), a function
\(u\colon\mathbb{R}\to\mathbb{R}\) is \emph{piecewise-Lipschitz} on \(B\) with
Lipschitz constant \(L\ge 0\) if \(|u(x)-u(y)|\le L\,|x-y|\) for every \(i\) and
all \(x,y\in I_i\). In other words, within each interval, a small change in the
payoff cannot cause a disproportionately large change in its utility. 
However, since no such restriction is imposed across different intervals, $u$ may be
discontinuous at their endpoints.
The special case \(L=0\) is the \emph{piecewise-constant} one, in
which \(u\) takes a single value \(u_i\in\mathbb{Q}\) on each \(I_i\).

\paragraph{Synthesis Problem Instance} %
A synthesis problem instance consists of an MDP \(\mathcal M\), a rational binning
\(B=\{I_1,\dots,I_n\}\), and a utility \(u\) that is piecewise-Lipschitz on
\(B\) with a given constant \(L\ge 0\).

For \(L=0\) the utility is piecewise-constant and we assume it is given
explicitly, as the list of its values \(u_1,\dots,u_n\in\mathbb{Q}\) on the intervals
of \(B\). For \(L>0\) the problem instance additionally specifies a rational precision
\(\varepsilon>0\), and the utility is presented by an \emph{evaluation oracle} $\hat u$:
on rational inputs \(x\) and \(\eta>0\) it returns a rational
\(\widehat u(x,\eta)\) with \(|\widehat u(x,\eta)-u(x)|\le\eta\), in time and
with output length polynomial in the  bit-length of \(x\) and in the
numerical value \(1/\eta\). Here and for the rest of the section, the bit-length of a rational number is the sum of the bit-lengths of its numerator and denominator when written in lowest terms.

For $L=0$, the task is to compute $V_u^*(\mathcal M)$ exactly together with an
optimal policy. For $L>0$, the task
is to compute an $\varepsilon$-accurate value and an $\varepsilon$-optimal
policy.

Recall that $\mu_{\mathcal M}$ is the reward bound of MDP $\mathcal M$ and $q_B$ is the largest denominator, in lowest terms, of endpoints of \(B\).
\begin{theorem}[Piecewise-constant utility]
\label{thm:pwc-utility}
For every synthesis problem instance with \(L=0\), the optimal expected utility
\(V_u^{*}(\mathcal M)\) is rational. Moreover, \(V_u^{*}(\mathcal M)\) and an optimal
finite-memory policy attaining it,  can be computed in time
polynomial in \(|\mathcal M|\), the number $n$ of bins, and the numerical values \(\mu_{\mathcal M}\) and \(q_B\).
\end{theorem}

\begin{theorem}[Piecewise-Lipschitz utility]
\label{thm:pwl-utility}
For every synthesis problem instance with $L>0$, an $\varepsilon$-optimal finite-memory
policy and a rational number $\widehat V$ satisfying
$|\widehat V-V_u^*(\mathcal M)|\leq\varepsilon$ can be computed in time
polynomial in $|\mathcal M|$, the number $n$ of bins, and
the numerical values $\mu_{\mathcal M}$, $q_B$, $\lceil L\rceil$, and
$\lceil 1/\varepsilon\rceil$.
\end{theorem}

For both theorems, the running time additionally depends polynomially on the maximum bit-length of the bin endpoints and, as applicable, the values $u_1,\dots,u_n$, $L$, and $\varepsilon$.

\paragraph{High-Level Idea} We now give a high-level overview of the proofs of the two theorems. The full details are presented in the following subsections.
Let $S$ and $d$ be the states and discount factor of MDP $\mathcal M$. For a piecewise-constant utility $u$, the expected utility for a policy $\theta$ decomposes as
\begin{equation}
\label{eq:pwc-exp}
V_u^\theta(\mathcal M)
=
\sum_{i=1}^n
u_i\,
\Pr\nolimits_{\mathcal M}^{\theta}[\DS_d\in I_i].
\end{equation}
Taking the product of $\mathcal M$ with the categorizer automaton $\mathcal A = (\mathcal T, F_1, \dots ,F_n)$ for $(\mu_{\mathcal M}, d, B)$ turns each event
$\DS_d\in I_i$ into a B\"uchi event $\Buchi(S \times F_i)$. The problem therefore becomes the
optimization of a weighted sum of B\"uchi probabilities in a finite MDP, which can be solved using the techniques of~\citet{DBLP:journals/tac/CourcoubetisY98}.

For a piecewise-Lipschitz utility, we refine the given binning and use the
evaluation oracle to construct a piecewise-constant utility $u'$ satisfying
$|u(x)-u'(x)|\leq\varepsilon/2$ for $x \in [-D_\mathcal{M}, D_\mathcal{M}]$. An optimal value and policy for
$u'$ are then $\varepsilon$-optimal for $u$.

\paragraph{Lower Bound}
The algorithms of Theorems~\ref{thm:pwc-utility}
and~\ref{thm:pwl-utility} are pseudo-polynomial: their running times depend on
the numerical values of $\mu_{\mathcal M}$ and $q_B$, rather than only on their
bit-lengths. We complement these upper bounds with a hardness result that already holds for three bins and $q_B=1$.

The \textsc{QSubsetSum} problem asks, given natural numbers
$k_1,\dots,k_N,T$, with $N$ even, whether
$\exists x_1\in\{0,1\}\,\forall x_2\in\{0,1\}\cdots
\exists x_{N-1}\in\{0,1\}\,\forall x_N\in\{0,1\}: \sum_{i=1}^N x_i k_i=T$. The problem is $\mathsf{PSPACE}$-complete
\citep[Lem.~4]{DBLP:journals/tcs/Travers06} and, as observed by
\citet[\S5]{DBLP:conf/icalp/HaaseK15}, can be represented by a
layered MDP with one level per number: the agent chooses whether to take \emph{reward} $k_i$
at odd level $i$, while a fair coin makes the choice at even level $i$. The agent has a policy under which the total reward equals $T$ almost surely exactly when the \textsc{QSubsetSum} instance is positive.

\begin{lemma}\label{lem:hard}
Deciding whether $V_u^*(\mathcal M)\geq 1$ for a piecewise-constant utility $u$
is $\mathsf{PSPACE}$-hard, already for discount factor $d=2$ and a binning $B$ with
three bins and integer endpoints, i.e. $q_B =1$.
\end{lemma}

\begin{proof}[Proof sketch]
Following the discount-balancing construction of
\citet{randour_percentile_2015}, in the layered MDP above set $d=2$ and multiply each reward at level $i$ by $d^{i-1}$. The discounted-sum payoff of any play is then exactly the
total reward in the original MDP. Finally, take
$B=\{(-\infty,T),\{T\},(T,+\infty)\}$ and assign utility $1$ to $\{T\}$ and utility $0$ to the other two bins. Thus, $V_u^*(\mathcal M)=1$ exactly when the \textsc{QSubsetSum} instance is positive. Full details are provided in the
supplementary material.
\end{proof}
We remark that this strongly suggests that there is no algorithm constructing categorizer automata for $(\mu, d, B)$ that runs in polynomial time in the bit-length of alphabet bound~$\mu$.

\subsection{Proof of Theorem~\ref{thm:pwc-utility}}
\label{sec:pwc-utility}
In this subsection we prove Thm.~\ref{thm:pwc-utility}. Fix an instance consisting of an MDP $\mathcal M=(S,\Act,s_{\mathrm{init}},P,r,d)$, a rational
binning $B=\{I_1,\dots,I_n\}$, and a piecewise-constant utility $u$ on \(B\) with value \(u_i\) on bin \(I_i\). 

Let \(\mathcal A=(\mathcal T,F_1,\dots,F_n)\) be the categorizer automaton for \((\mu_{\mathcal M},d,B)\), with \(\mathcal T=(Q,\Sigma_{\mu_{\mathcal M}},q_{\mathrm{init}},\delta)\).

\paragraph{Product with Categorizer Automaton}
The product $\mathcal M\otimes\mathcal A$ is an MDP that records the current states of both the MDP $\mathcal M$ and the categorizer automaton $\mathcal A$. It has state space $S\times Q$, action set $\Act$, initial state $(s_{\mathrm{init}},q_{\mathrm{init}})$, and transition probability $P^{\otimes}((s,q),a,(s',q'))=P(s,a,s')$ if $q'=\delta(q,r(s,a))$, and $0$ otherwise. 

Because the categorizer automaton $\mathcal A$ is deterministic, every play $\pi=s_0a_0s_1a_1\cdots$ of $\mathcal M$ has a unique corresponding play
$\pi^\otimes=(s_0,q_0)a_0(s_1,q_1)a_1\dots $ of $\mathcal M\otimes\mathcal A$, where $q_0=q_{\mathrm{init}}$ and $q_{k+1}=\delta(q_k,r(s_k,a_k))$ for every $k\geq 0$. Conversely, the projection of every play of $\mathcal M\otimes\mathcal A$ on its $S$-component yields a unique play of $\mathcal M$. 

The same correspondence holds between finite histories of $\mathcal M$ and $\mathcal M \otimes \mathcal A$ and extends to their policies: given a policy $\theta$ on $\mathcal M$, define the policy $\theta^\otimes$ on $\mathcal M \otimes \mathcal A$ by $\theta^\otimes(h^\otimes,a)=\theta(h,a)$ for every pair of corresponding histories $h$ and $h^\otimes$. Conversely, every policy on the product determines a policy on $\mathcal M$. Corresponding histories use the same actions and transition probabilities, and hence corresponding measurable events have the same probability.

\paragraph{Reduction} By construction, a play $\pi$ of $\mathcal M$ satisfies \(\DS_d(\pi)\in I_i\) iff the corresponding play \(\pi^{\otimes}\) of $\mathcal M \otimes \mathcal A$ visits \(S \times F_i\) infinitely often. With~\eqref{eq:pwc-exp} this gives, for corresponding \(\theta\) and \(\theta^{\otimes}\), the expected utility $V_u^{\theta}(\mathcal M)$ is equal to the weighted sum $\sum_{i=1}^{n}u_i
  \Pr\nolimits_{\mathcal M\otimes\mathcal A}^{\theta^{\otimes}}
  \left[\Buchi(S\times F_i)\right]$ of B\"uchi probabilities.
When all $u_i$ are nonnegative, maximizing the right-hand side can be done using the following straightforward consequence of~\cite{DBLP:journals/tac/CourcoubetisY98}.
\begin{proposition}
\label{prop:weighted-buchi}
For a finite MDP \(\mathcal N\) with state space $Z$, sets \(Z_1,\dots,Z_k \subseteq Z\), and weights
\(v_1,\dots,v_k \in \Q_{\ge 0}\), the optimal value
\(
  \sup_{\theta \in \Theta_{\mathcal N}}\;\sum_{i=1}^{k}v_i\cdot
     \Pr\nolimits^{\theta}_{\mathcal{N}} [\Buchi(Z_i)]
\) is rational. Moreover, this value and a finite-memory policy attaining it can be computed in time polynomial in \(|\mathcal N|\), \(k\) and the maximum bit-length of the weights.
\end{proposition} %
Indeed, in the corresponding theorem in~\cite[Thm.~5.1, \S~7]{DBLP:journals/tac/CourcoubetisY98} the events $\Buchi(Z_i)$ are given by automata that must first be composed with the MDP $\mathcal N$, and the running time is polynomial in the size of that product. Here, each event is already a B\"uchi condition on the states of $\mathcal N$, so the composition step is not needed and the rest of their proof applies.

If some utility values are negative, we shift all values by a common constant before applying Prop.~\ref{prop:weighted-buchi}. Let $c=\max\{0,-\min_i u_i\}$ and define $\bar u(x)=u(x)+c$. Then $\bar u$ is piecewise-constant on $B$, with value $\bar u_i=u_i+c\geq 0$ on $I_i$. By taking expectations, for every policy $\theta$, $V_{\bar u}^{\theta}(\mathcal M)=V_u^{\theta}(\mathcal M)+c$. Consequently, $u$ and $\bar u$ have the same optimal policies, and $V_{\bar u}^*(\mathcal M)=V_u^*(\mathcal M)+c$.

Applying Prop.~\ref{prop:weighted-buchi} to $\mathcal M\otimes\mathcal A$, with B\"uchi sets $S\times F_1,\dots,S\times F_n$ and nonnegative rational weights $\bar u_1,\dots,\bar u_n$, computes $V_{\bar u}^*(\mathcal M)$ and an optimal finite-memory policy  $\theta^\otimes$ on the product. Subtracting $c$ from the computed value gives $V_u^*(\mathcal M)$. Transferring the policy $\theta^\otimes$ to $\theta$ of $\mathcal M$ yields an
optimal finite-memory policy for $u$: its memory maintains the current state of
$\mathcal A$ together with the memory used by $\theta^\otimes$.
The size of $\mathcal M\otimes\mathcal A$ is polynomial in
$|\mathcal M|$ and $|\mathcal A|$. Thm.~\ref{thm:size} and
Prop.~\ref{prop:weighted-buchi} therefore give the complexity bound
of Thm.~\ref{thm:pwc-utility}.

\subsection{Proof of Theorem~\ref{thm:pwl-utility}}
\label{sec:pwl-utility}
In this subsection, we prove Thm.~\ref{thm:pwl-utility}. Fix an instance
consisting of an MDP $\mathcal M=(S,\Act,s_{\mathrm{init}},P,r,d)$, a rational
binning $B=\{I_1,\dots,I_n\}$, a
rational precision $\varepsilon>0$, and a piecewise-Lipschitz utility $u$ on $B$ with
Lipschitz constant $L>0$, represented by an evaluation oracle $\widehat u$ .

\paragraph{Piecewise-Constant Approximation}
Let $\hat{L}=\lceil L\rceil$, $E=\lceil 1/\varepsilon\rceil$, $\tau=1/(4\hat{L} E)$, and $W=\lceil D_{\mathcal M}/\tau\rceil+1$. Thus, $L\leq\hat{L}$, $1/E\leq\varepsilon$, and $[-D_{\mathcal M},D_{\mathcal M}] \subseteq [-W\tau,W\tau]$.

Construct a binning $B'=\{J_1,\dots,J_r\}$, using the endpoints of $B$ and the grid points $k\tau$, for $k \in \{-W,\dots, W\}$, such that every bin intersecting $[-D_{\mathcal M},D_{\mathcal M}]$ has length at most $\tau$ and every bin $J_i$ is contained in a bin of $B$.

For every bin $J_i$ intersecting $[-D_{\mathcal M},D_{\mathcal M}]$, choose a rational point $z_i\in J_i\cap[-D_{\mathcal M},D_{\mathcal M}]$ whose bit-length is
polynomial in the bit-lengths of the endpoints of $J_i$. Query the evaluation oracle at $z_i$ with precision $\eta=1/(4E)$, and let $u'_i=\widehat u(z_i,\eta)$. Define $u'(x)=u'_i$ for every $x\in J_i$. On bins disjoint from $[-D_{\mathcal M},D_{\mathcal M}]$, define $u'$ to be $0$.

The resulting utility $u'$ is piecewise-constant on $B'$ and satisfies
$|u(x)-u'(x)|\leq\varepsilon/2$ for every
$x\in[-D_{\mathcal M},D_{\mathcal M}]$. Indeed, let $J_i$ be the bin
containing $x$. Since $x,z_i\in J_i$, $J_i$ has length at most
$\tau$, and $J_i$ is contained in a bin of $B$, the Lipschitz condition gives
$|u(x)-u(z_i)|\leq L\tau$. The oracle guarantee gives
$|u(z_i)-u'_i|\leq\eta$. Therefore, $|u(x)-u'(x)|
\leq |u(x)-u(z_i)|+|u(z_i)-u'_i|
\leq L\tau+\eta
\leq 1/(2E)
\leq {\varepsilon}/{2}.$

\paragraph{Reduction and Complexity} Taking expectations, for every policy $\theta$ of $\mathcal M$,
\(V_{u'}^\theta(\mathcal M)-\varepsilon/2 \le V_u^\theta(\mathcal M) \le V_{u'}^\theta(\mathcal M)+\varepsilon/2. \) Thus, the optimal value $V_{u'}^*(\mathcal M)$ satisfies $|V_{u}^*(\mathcal M) -V_{u'}^*(\mathcal M) | \le \varepsilon/2$, and an optimal policy $\theta^*$ for $u'$ is $\varepsilon$-optimal for $u$ (since $V_u^*(\mathcal M)
\leq V_{u'}^{\theta^*}(\mathcal M)+\varepsilon/2
\leq V_u^{\theta^*}(\mathcal M)+\varepsilon$). 
Thm.~\ref{thm:pwc-utility} computes $V_{u'}^*(\mathcal M)$ and such a
policy $\theta^*$.

Since $d\geq 2$, we have $D_{\mathcal M}\leq 2\mu_{\mathcal M}$, and hence $W=\mathcal O(\mu_{\mathcal M}\hat{L} E)$. Therefore, the number $r$ of bins of $B'$ is on the order of $n+\mathcal O(\mu_{\mathcal M}\hat{L} E)$. 
Moreover, since every new endpoint is of the form $k/(4\hat{L} E)$, the largest
denominator $q_{B'}$ is at most $\max\{q_B,4\hat{L} E\}$.

Constructing $u'$ requires $r$ oracle calls, each with precision $\eta=1/(4E)$. By the choice of the query points and the oracle assumption, the total running time of these calls and the bit-lengths of their outputs are polynomial in $r$, the bit-lengths of the endpoints of $B'$, and the numerical value $E$. Together with the bounds on $r$ and $q_{B'}$ above, Thm.~\ref{thm:pwc-utility} gives the running-time bound claimed in Thm.~\ref{thm:pwl-utility}.

\section{Conclusion}\label{sec:discussion}
We introduced categorizer automata and showed how to construct them in time and space linear (rather than exponential) in the number of bins. We then applied this construction to obtain algorithms for synthesizing policies that
maximize the expected utility of the payoff for a broad class of (possibly) discontinuous
utilities. We view this as one instance of a broader use of categorizer automata. 
A natural next step is to synthesize policies in MDPs that optimize \emph{soft} quantitative preferences, expressed through discounted-sum payoffs, subject to \emph{hard} qualitative requirements, expressed by LTL specifications. These two objectives are typically handled using different algorithmic techniques, see~\cite{Puterman94} and~\cite{PMC-book}. A categorizer automaton bridges the mismatch by translating the quantitative payoff into $\omega$-regular conditions. 

Finally, our construction of categorizer automata, like the comparator automata before it, requires the discount factor to be an integer. One potential route to non-integer
discount factors is an approximate categorizer automaton, in the spirit
of~\citet{randour_percentile_2015}, that is required to identify the correct bin only when the payoff is at a distance greater than $\varepsilon$ from every
endpoint. This relaxation introduces only a vanishing utility error when the
utility is continuous at the endpoints. At a discontinuity, however, it cannot
recover the guarantees of Theorems~\ref{thm:pwc-utility} and~\ref{thm:pwl-utility}, and new ideas are needed.

\bibliography{main}

\begin{thebibliography}{22}
\providecommand{\natexlab}[1]{#1}

\bibitem[{Baier and Katoen(2008)}]{PMC-book}
Baier, C.; and Katoen, J. 2008.
\newblock \emph{Principles of model checking}.
\newblock {MIT} Press.

\bibitem[{Bansal, Chatterjee, and Vardi(2021)}]{DBLP:conf/tacas/BansalCV21}
Bansal, S.; Chatterjee, K.; and Vardi, M.~Y. 2021.
\newblock On Satisficing in Quantitative Games.
\newblock In Groote, J.~F.; and Larsen, K.~G., eds., \emph{Tools and Algorithms
  for the Construction and Analysis of Systems {TACAS} 2021}, Lecture Notes in
  Computer Science, 20--37. Springer.

\bibitem[{Bansal, Chaudhuri, and Vardi(2018)}]{DBLP:conf/fossacs/BansalCV18}
Bansal, S.; Chaudhuri, S.; and Vardi, M.~Y. 2018.
\newblock Comparator Automata in Quantitative Verification.
\newblock In \emph{FoSSaCS}, volume 10803 of \emph{Lecture Notes in Computer
  Science}, 420--437. Springer.

\bibitem[{Bansal, Chaudhuri, and Vardi(2022)}]{bansal2022comparator}
Bansal, S.; Chaudhuri, S.; and Vardi, M.~Y. 2022.
\newblock Comparator automata in quantitative verification.
\newblock \emph{Logical Methods in Computer Science}, 18.

\bibitem[{Bansal et~al.(2022)Bansal, Kavraki, Vardi, and
  Wells}]{rational-discount-AAAI}
Bansal, S.; Kavraki, L.; Vardi, M.; and Wells, A. 2022.
\newblock Synthesis from Satisficing and Temporal Goals.
\newblock \emph{Proceedings of the AAAI Conference on Artificial Intelligence},
  36: 9679--9686.

\bibitem[{Boker and Henzinger(2014)}]{BokerH14}
Boker, U.; and Henzinger, T.~A. 2014.
\newblock Exact and Approximate Determinization of Discounted-Sum Automata.
\newblock \emph{Log. Methods Comput. Sci.}, 10(1).

\bibitem[{Bäuerle and Rieder(2014)}]{INFORMS2014}
Bäuerle, N.; and Rieder, U. 2014.
\newblock More Risk-Sensitive Markov Decision Processes.
\newblock \emph{Mathematics of Operations Research}, 39(1): 105--120.

\bibitem[{Courcoubetis and
  Yannakakis(1998)}]{DBLP:journals/tac/CourcoubetisY98}
Courcoubetis, C.; and Yannakakis, M. 1998.
\newblock Markov decision processes and regular events.
\newblock \emph{{IEEE} Trans. Autom. Control.}, 43(10): 1399--1418.

\bibitem[{de~Alfaro, Henzinger, and
  Majumdar(2003)}]{DBLP:conf/icalp/AlfaroHM03}
de~Alfaro, L.; Henzinger, T.~A.; and Majumdar, R. 2003.
\newblock Discounting the Future in Systems Theory.
\newblock In \emph{{ICALP}}, volume 2719 of \emph{Lecture Notes in Computer
  Science}, 1022--1037. Springer.

\bibitem[{Etessami et~al.(2008)Etessami, Kwiatkowska, Vardi, and
  Yannakakis}]{lmcs:990}
Etessami, K.; Kwiatkowska, M.; Vardi, M.~Y.; and Yannakakis, M. 2008.
\newblock Multi-Objective Model Checking of Markov Decision Processes.
\newblock \emph{Logical Methods in Computer Science}, Volume 4, Issue 4: 8.

\bibitem[{Haase and Kiefer(2015)}]{DBLP:conf/icalp/HaaseK15}
Haase, C.; and Kiefer, S. 2015.
\newblock The Odds of Staying on Budget.
\newblock In \emph{{ICALP} {(2)}}, volume 9135 of \emph{Lecture Notes in
  Computer Science}, 234--246. Springer.

\bibitem[{Kadota, Kurano, and Yasuda(1998)}]{General-Utility}
Kadota, Y.; Kurano, M.; and Yasuda, M. 1998.
\newblock On the General Utility of Discounted Markov Decision Processes.
\newblock \emph{International Transactions in Operational Research}, 5(1):
  27--34.

\bibitem[{Mas-Colell et~al.(1995)Mas-Colell, Whinston, Green
  et~al.}]{mas1995microeconomic}
Mas-Colell, A.; Whinston, M.~D.; Green, J.~R.; et~al. 1995.
\newblock \emph{Microeconomic theory}, volume~1.
\newblock Oxford university press New York.

\bibitem[{Puterman(1994)}]{Puterman94}
Puterman, M.~L. 1994.
\newblock \emph{Markov Decision Processes: Discrete Stochastic Dynamic
  Programming}.
\newblock Wiley Series in Probability and Statistics. Wiley.

\bibitem[{Rajasekaran, Bansal, and Vardi(2023)}]{rajasekaran_multi-agent_2023}
Rajasekaran, S.; Bansal, S.; and Vardi, M.~Y. 2023.
\newblock Multi-Agent Systems with Quantitative Satisficing Goals.
\newblock In \emph{{IJCAI}}, 280--288. ijcai.org.

\bibitem[{Randour, Raskin, and Sankur(2015)}]{randour_percentile_2015}
Randour, M.; Raskin, J.-F.; and Sankur, O. 2015.
\newblock Percentile {Queries} in {Multi}-dimensional {Markov} {Decision}
  {Processes}.
\newblock In Kroening, D.; and Păsăreanu, C.~S., eds., \emph{Computer {Aided}
  {Verification}}, 123--139. Cham: Springer International Publishing.
\newblock ISBN 978-3-319-21690-4.

\bibitem[{Simon(1956)}]{simon1956rational}
Simon, H.~A. 1956.
\newblock Rational choice and the structure of the environment.
\newblock \emph{Psychological Review}, 63(2): 129--138.

\bibitem[{Travers(2006)}]{DBLP:journals/tcs/Travers06}
Travers, S.~D. 2006.
\newblock The complexity of membership problems for circuits over sets of
  integers.
\newblock \emph{Theor. Comput. Sci.}, 369(1-3): 211--229.

\bibitem[{White(1993)}]{white_minimizing_1993}
White, D.~J. 1993.
\newblock Minimizing a {Threshold} {Probability} in {Discounted} {Markov}
  {Decision} {Processes}.
\newblock \emph{Journal of Mathematical Analysis and Applications}, 173(2):
  634--646.

\bibitem[{Wu and Lin(1999)}]{wu_minimizing_1999}
Wu, C.; and Lin, Y. 1999.
\newblock Minimizing {Risk} {Models} in {Markov} {Decision} {Processes} with
  {Policies} {Depending} on {Target} {Values}.
\newblock \emph{Journal of Mathematical Analysis and Applications}, 231(1):
  47--67.

\bibitem[{Wu and Xu(2023)}]{DBLP:journals/corr/abs-2311-13589}
Wu, Z.; and Xu, R. 2023.
\newblock Risk-sensitive Markov Decision Process and Learning under General
  Utility Functions.
\newblock \emph{CoRR}, abs/2311.13589.

\bibitem[{Xu and Mannor(2011)}]{xu_probabilistic_2011}
Xu, H.; and Mannor, S. 2011.
\newblock Probabilistic goal {Markov} decision processes.
\newblock In \emph{Proceedings of the {Twenty}-{Second} international joint
  conference on {Artificial} {Intelligence} - {Volume} {Volume} {Three}},
  {IJCAI}'11, 2046--2052. Barcelona, Catalonia, Spain: AAAI Press.
\newblock ISBN 978-1-57735-515-1.

\end{thebibliography}
\appendix

\section{Supplementary Material}
\subsection{Categorizer Automata}
\begin{proof}[Proof of Proposition~\ref{prop:gaps}]
Fix $j$ and $w$, and abbreviate $g^k = g_j(w[1..k])$ and
$\sigma^k = \DS_d(w[k{+}1..]) \in [-D, D]$. Since
$\DS_d(w) = \DS_d(w[1..k]) + d^{-k}\sigma^k$, the closed
form~\eqref{eq:closedform} gives
\begin{equation}
  t_j - \DS_d(w) \;=\; d^{-k}(g^k - \sigma^k)
  \qquad \text{for every } k \ge 0.
  \label{eq:gapwitness}
\end{equation}
If $g^k < -D$ for some $k$, then $g^k - \sigma^k < 0$ because
$\sigma^k \ge -D$, so $\DS_d(w) > t_j$ by~\eqref{eq:gapwitness};
symmetrically, $g^k > D$ for some $k$ implies $\DS_d(w) < t_j$. If
$g^k \in [-D, D]$ for all $k$, then $|g^k - \sigma^k| \le 2D$, so
$|t_j - \DS_d(w)| \le 2D\,d^{-k}$ for all $k$, and the right-hand side
tends to $0$ as $k$ tends to infinity because $d > 1$; hence
$\DS_d(w) = t_j$.
 
The three hypotheses above say that $t_j$ is resolved to $\bot$ after
some prefix of $w$, resolved to $\top$ after some prefix of $w$, and
active after every prefix of $w$, respectively. They are exhaustive,
since $g^k \notin [-D, D]$ means $g^k < -D$ or $g^k > D$, and pairwise
disjoint: the third excludes the other two by definition, and the first
two exclude each other because their conclusions $\DS_d(w) > t_j$ and
$\DS_d(w) < t_j$ do. Since the three conclusions are likewise exhaustive and pairwise disjoint, each implication is in fact an equivalence.
\end{proof}

\begin{proof}[Proof of Proposition~\ref{prop:integer-d}]
\textsf{P1}. We claim that $g_j(u) \in \frac{1}{q_j}\mathbb{Z}$ for every
$u \in \Sigma_\mu^*$. Indeed, by~\eqref{eq:closedform},
\[
  g_j(u) \;=\; \frac{d^{|u|} p_j}{q_j}
               - \sum_{i=1}^{|u|} u[i]\, d^{|u| - i + 1},
\]
where the sum is an integer because $d$ is an integer and
$|u| - i + 1 \ge 1$ for $1 \le i \le |u|$. If $t_j$ is active after $u$
then moreover $g_j(u) \in [-D, D]$, so $g_j(u) \in G_j$.

\noindent \textsf{P2}. Since $D - \mu = \mu/(d-1)$, the bound $D$ is the gap update at the integer bound $\mu$:
\begin{equation}
  d\,(D - \mu) = D
  \qquad\text{and}\qquad
  d\,(-D + \mu) = -D .
  \label{eq:fixpoint}
\end{equation}
Let $a \in \Sigma_\mu$, so $-\mu \le a \le \mu$. If $t_j$ is resolved to
$\bot$ after $u$, i.e.\ $g_j(u) < -D$, then $g_j(u) - a < -D + \mu$, and
multiplying by $d > 0$ and using~\eqref{eq:fixpoint} gives
\[
  g_j(u \cdot a) \;=\; d\,(g_j(u) - a)
                 \;<\; d\,(-D + \mu) \;=\; -D,
\]
so $t_j$ is resolved to $\bot$ after $u \cdot a$. Symmetrically, if
$g_j(u) > D$ then $g_j(u) - a > D - \mu$ and
$g_j(u \cdot a) > d\,(D - \mu) = D$, so $t_j$ remains resolved
to $\top$.
\end{proof}

\paragraph{Formal Definition of $q_\text{init}$}

Recall that by~\eqref{eq:gaps} we have $g_i(\varepsilon) = t_i$. Let
$A_{\text{init}} = \{i \mid t_i \in [-D, D]\}$ and $\ell_{\mathrm{init}}=\max\{i\mid t_i<-D\},$ where we use the convention $\max\emptyset=0$. Accordingly,
\[
  q_{\mathrm{init}} =
  \begin{cases}
    (j, t_j, 0),\ j = \min A_{\text{init}}
      & \text{if } |A_{\text{init}}| \ge 2, \\
    (j, t_j)
      & \text{if } A_{\text{init}} = \{j\}, \\
    \ell_{\text{init}}
      & \text{if } A_{\text{init}} = \emptyset .
  \end{cases}
\]
\paragraph{Formal Definition of $\delta$}
Let $s \in Q$ be a state and let $r \in \Sigma_\mu$ be an input symbol. The transition $\delta(s,r)$ is defined as follows:
\begin{itemize}
    \item $\mathbf{s \in Q_a}$  Let $s=(j,g,k)$. We call $s$ \emph{consistent} if
$g-d^k t_j\in\mathbb Z$. If $s$ is not consistent, set
$\delta(s,r)=s$.

Every state in $Q_a$ reachable from $q_{\mathrm{init}}$ is consistent:
if it is reached after a word $u$ of length $k$, then by~\eqref{eq:closedform},
\[
g-d^k t_j=-d^k\DS_d(u)\in\mathbb Z.
\]
Hence, the self-loops added above do not affect any run from
$q_{\mathrm{init}}$.

Suppose now that $s$ is consistent. First reconstruct
$\widehat g_h=g+d^k(t_h-t_j)$ for $1\leq h\leq m$, as in P1 of
Lemma~\ref{lem:facts}, and then apply update~\eqref{eq:gaps} to obtain
$\widehat g'_h=d(\widehat g_h-r)$. Let
\[
  A' = \{ i \mid \widehat{g}_i' \in [-D, D] \},
  \quad
  \ell' = \max \{ i \mid \widehat{g}_i' < -D \} .
\]
The successor $s' = \delta(s, r)$ is determined by $|A'|$:
\begin{itemize}
\item if $|A'| \ge 2$, then
  $s' = (j', \widehat{g}_{j'}', k{+}1)$ with
  $j' = \min A'$;
\item if $A' = \{h\}$, then
  $s' = (h, \widehat{g}_h')$;
\item if $A' = \emptyset$, then
  $s' = \ell'$.
\end{itemize}
\item  $\mathbf{s \in Q_b}$ Let $s =(j, g)$, let $g' = d\,(g - r)$ as
in~\eqref{eq:gaps}, and set
\[
  \delta(s, r) =
  \begin{cases}
    (j, g') & \text{if } g' \in [-D, D],\\
    j       & \text{if } g' < -D,\\
    j - 1   & \text{if } g' > D.
  \end{cases}
\]
\item $\mathbf{s \in Q_c}$  $\delta(s, r) = s$.
\end{itemize}

\subsection{Lower Bound}

\begin{proof}[Proof of Lemma~\ref{lem:hard}]
We give the details of the reduction from \textsc{QSubsetSum} described in the proof sketch. Fix an instance $k_1,\dots,k_N,T$, where $N$ is even. We construct
the MDP $\mathcal M$ shown in Figure~\ref{fig:hardness} and set its discount
factor to $d=2$. For every odd $i$, the MDP has one state $c_i$, while for
every even $i$, it has two states $c_i^0$ and $c_i^1$. 
It also has an absorbing state $s_{\mathrm{sink}}$, and its initial state is
$c_1$. Thus, $\mathcal M$ has $3N/2+1$ states.

At an odd-level state $c_i$, the actions $\mathsf{take}$ and $\mathsf{skip}$
are enabled. Under either action, the next state is $c_{i+1}^0$ or
$c_{i+1}^1$, each with probability $1/2$. The rewards of these actions are
\[
r(c_i,\mathsf{take})=k_i d^{i-1}
\qquad\text{and}\qquad
r(c_i,\mathsf{skip})=0.
\]
At an even-level state $c_i^b$, where $b\in\{0,1\}$, exactly one action is
enabled, and its reward is $b\,k_i d^{i-1}$. For $i<N$, this action leads with
probability one to $c_{i+1}$; from level $N$, it leads to
$s_{\mathrm{sink}}$. The only action enabled at $s_{\mathrm{sink}}$ has reward
$0$ and returns to $s_{\mathrm{sink}}$.

For a play $\pi$, let $x_i(\pi)\in\{0,1\}$ indicate whether $k_i$ is taken.
At an odd level, this is determined by whether the policy chooses
$\mathsf{take}$ or $\mathsf{skip}$. At an even level, it is determined by
whether the play enters $c_i^1$ or $c_i^0$. The reward associated with level
$i$ is collected at step $i-1$. Therefore,
\[
\DS_d(\pi)
 =\sum_{i=1}^{N}
   \frac{x_i(\pi)k_i d^{i-1}}{d^{i-1}}
 =\sum_{i=1}^{N}x_i(\pi)k_i.
\]
Thus, the scaling by $d^{i-1}$ exactly cancels the discounting. The rewards
may be exponentially large in value, but multiplying $k_i$ by
$2^{i-1}$ increases its binary encoding length by only $i-1$. Hence,
$\mathcal M$ can be constructed in time polynomial in the size of the
\textsc{QSubsetSum} instance.

Consider the three-bin binning $B =\{(-\infty,T),\{T\},(T,+\infty)\}$~\footnote{Binning $B=\{(-\infty,T-1],(T-1, T+1),[T+1,+\infty)\}$ also works as the discounted-sum payoff is always an integer.}
and the piecewise-constant utility $u$ that is $1$ on $\{T\}$ and $0$ on the
other two bins. All endpoints are integers, so $q_B=1$. Viewing each $x_i$ as
a random variable on plays, for every policy $\theta$ we have
\[
V_u^\theta(\mathcal M)
 =\Pr_{\mathcal M}^{\theta}
   \left[\sum_{i=1}^{N}x_i k_i=T\right].
\]
Since $u$ takes values in $\{0,1\}$, we have $V_u^*(\mathcal M)\leq 1$.
Moreover, only the first $N$ steps affect the payoff, so an optimal policy
exists. It follows that $V_u^*(\mathcal M)\geq 1$ exactly when some policy
makes the sum equal to $T$ almost surely. By~\cite{DBLP:conf/icalp/HaaseK15} such a policy exists exactly when the \textsc{QSubsetSum} instance is positive.

The reduction is polynomial and uses the fixed discount factor $d=2$, three
bins, and integer endpoints. The claimed $\mathsf{PSPACE}$-hardness follows.
\end{proof}

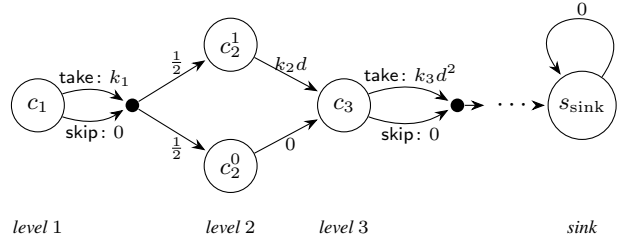
\begin{figure}[t]
\centering
\begin{tikzpicture}[
  >={Stealth[length=1.8mm]},
  shorten >=1pt,
  st/.style={
    circle,
    draw,
    minimum size=7mm,
    inner sep=0pt,
    font=\small
  },
  sink/.style={
    circle,
    draw,
    minimum size=9mm,
    inner sep=1pt,
    font=\small
  },
  rw/.style={font=\scriptsize,inner sep=1.5pt},
  pb/.style={font=\scriptsize,inner sep=1.5pt},
  lv/.style={font=\scriptsize\itshape,inner sep=1.5pt},
  pr/.style={
    circle,
    fill=black,
    minimum size=1.8mm,
    inner sep=0pt
  }
]
  \node[st]   (c1)   at (0,0)        {$c_1$};
  \node[pr]   (b1)   at (1.25,0)     {};
  \node[st]   (u2)   at (2.55,0.8)   {$c_2^1$};
  \node[st]   (d2)   at (2.55,-0.8)  {$c_2^0$};
  \node[st]   (c3)   at (4.05,0)     {$c_3$};
  \node[pr]   (b3)   at (5.55,0)     {};
  \node       (dots) at (6.3,0)      {$\cdots$};
  \node[sink] (end)  at (7.2,0)      {$s_{\mathrm{sink}}$};

  \draw[->]
    (c1) to[bend left=25]
    node[rw,above]{$\mathsf{take}\colon k_1$} (b1);
  \draw[->]
    (c1) to[bend right=25]
    node[rw,below]{$\mathsf{skip}\colon 0$} (b1);

  \draw[->]
    (b1) -- node[pb,above,pos=0.55]{$\tfrac12$} (u2);
  \draw[->]
    (b1) -- node[pb,below,pos=0.55]{$\tfrac12$} (d2);

  \draw[->]
    (u2) -- node[rw,above,pos=0.55]{$k_2d$} (c3);
  \draw[->]
    (d2) -- node[rw,below,pos=0.55]{$0$} (c3);

  \draw[->]
    (c3) to[bend left=25]
    node[rw,above]{$\mathsf{take}\colon k_3d^2$} (b3);
  \draw[->]
    (c3) to[bend right=25]
    node[rw,below]{$\mathsf{skip}\colon 0$} (b3);

  \draw[->] (b3) -- (dots);
  \draw[->] (dots) -- (end);
  \draw[->]
    (end) to[out=55,in=125,looseness=6]
    node[rw,above]{$0$} (end);

  \node[lv] at (0,-1.6)    {level $1$};
  \node[lv] at (2.55,-1.6) {level $2$};
  \node[lv] at (4.05,-1.6) {level $3$};
  \node[lv] at (7.2,-1.6)  {sink};
\end{tikzpicture}
\caption{The layered MDP $\mathcal M$ used in the lower bound. At odd
levels, edge labels give the action and its reward. At even levels and at
$s_{\mathrm{sink}}$, only one action is enabled, so only its reward is shown.
Filled dots denote fair probabilistic branching. The repeated levels after
level $3$ are omitted.}
\label{fig:hardness}
\end{figure}
\end{document}